\documentclass[sigconf]{acmart}
\AtBeginDocument{%
  \providecommand\BibTeX{{%
    \normalfont B\kern-0.5em{\scshape i\kern-0.25em b}\kern-0.8em\TeX}}}

\setcopyright{acmcopyright}
\copyrightyear{2022}
\acmYear{2022}
\acmDOI{XXXXXXX.XXXXXXX}

\acmConference[ICAIF '22]{Workshop on Explainable AI in Finance}{November 2 2022}{New York, and Online}

\begin{document}

\title{Feature Importance for Time Series Data: Improving KernelSHAP}

\author{Mattia Jacopo Villani}
\email{mattia.villani@jpmorgan.com}
\affiliation{%
  \institution{J.P. Morgan AI Research}
  \city{London}
  \country{UK}
}
\affiliation{%
  \institution{ King's College London}
  \city{London}
  \country{UK}
}

\author{Joshua Lockhart}
\email{joshua.lockhart@jpmorgan.com}
\affiliation{%
  \institution{J.P. Morgan AI Research}
  \city{London}
  \country{UK}
}

\author{Daniele Magazzeni}
\email{daniele.magazzeni@jpmorgan.com}
\affiliation{%
  \institution{J.P. Morgan AI Research}
  \city{London}
  \country{UK}
}

\begin{abstract}
  Feature importance techniques have enjoyed widespread attention in the explainable AI literature as a means of determining how trained machine learning models make their predictions. We consider Shapley value based approaches to feature importance, applied in the context of time series data. We present closed form solutions for the SHAP values of a number of time series models, including VARMAX. We also show how KernelSHAP can be applied to time series tasks, and how the feature importances that come from this technique can be combined to perform ``event detection''. Finally, we explore the use of Time Consistent Shapley values for feature importance.
\end{abstract}



\maketitle

\section{Introduction}

As deep learning models become more widely used, the importance of explaining their predictions increases. Efforts in Explainable Artificial Intelligence (XAI) research have led to the development of frameworks for interpreting model behaviour, including counterfactuals \cite{wachter2017counterfactual}, \cite{karimi2020model}, local surrogates \cite{ribeiro2016should} and feature importance \cite{lundberg2017unified}, \cite{simonyan2013deep}, \cite{smilkov2017smoothgrad}. This growing range of techniques all aim at producing information that can assist model-makers in establishing the soundness of their models. 
Indeed, the idea of the consumer's right to an explanation on an AI-driven decision is driving an increased focus on criteria for explainable artificial intelligence
\cite{gunning2019xai}, \cite{dovsilovic2018explainable}. In general, explanations are key to ensure that model decisions are compatible with ethical standards. 

Beyond their broader social scope, producing explanations can help to mitigate model risk, by enabling a better understanding of the limitations of a model. In fact, XAI techniques are often employed to `debug' models. Research in AI is therefore necessary to enable the safer deployment of models such as Artificial Neural Networks (ANNs), which offer high predictive performance on many tasks at the expense of interpretability.

Feature importance scores provide a first approximation of model behaviour, revealing which features were influential on the model predictions. This can either be in the context of a single prediction (falling within the scope of \textit{local} explainability) or a statement of how the model behaves more generally: either for a given data set or for any data set (\textit{global} explainability) \cite{dovsilovic2018explainable}. Many of these techniques are \textit{model agnostic}, meaning that they are not tied to any particular model architecture.
However, we will argue that these techniques can and should be refined to the particular model on which they are being implemented. 

Here we examine the particular case of the time series domain. In this context, we explore how the particular characteristics affect the implementation of Shapley Additive Explanations (SHAP) \cite{lundberg2017unified}. As we explore later, there are several important constraints in deploying KernelSHAP in the time series domain. Firstly, KernelSHAP makes use of the coefficients of a linear model that has been fit to perturbations of the data. We ascertain that the results guaranteeing the convergence of these coefficients to the SHAP values are indeed preserved in the time series domain, where instead of a linear model we need to fit a Vector Autoregressive (VAR) model.

Secondly, time series models often take windows of data; in the presence of long input windows the computation of KernelSHAP becomes impossible. As we will see later, this is due to numerical underflow in the normalisation coefficient present in the calculations required for KernelSHAP.
Indeed, in practice, it is common to find very long lookback windows. Thirdly, KernelSHAP assumes independence of features, which is emphatically an exception rather than the norm in the time series domain. It is common for time series to possess autocorrelations, meaning correlations between values attained by a given variable at different time steps.

In this paper we present several extensions to the KernelSHAP algorithm to compute SHAP values, for the time series domain, addressing the aforementioned areas, as well as a method for detecting \textit{events} through SHAP. Moreover, we provide explicit SHAP values for broadly used time series models AR, MA, ARMA, VARMA, VARMAX. Our contributions, in the structure of the paper, are:

\begin{enumerate}
    \item Proof of suitability for the application of KernelSHAP in the context of time series data. Our proof builds on the approximation of SHAP with linear models in KernelSHAP, extending it to the time series domain through VAR models, whose calibration also approximates SHAP. We call this alteration VARSHAP. 
    \item Explicit SHAP values for widely used time series models: autoregressive, moving average and vector models of the same with exogenous variables. 
    \item We present Time Consistent SHAP Values, a new feature importance technique for the time series domain, which leverages the temporal component of the problem in order to cut down the sample space involved in the KernelSHAP computation.
    \item An aggregation technique which is able to capture surges in feature importance across time steps, which we call \textit{event detection}. 
\end{enumerate}

\subsection{Related Work}

Shapley Additive Explanations (SHAP Values) \cite{lundberg2017unified} are a broadly used feature importance technique that leverage an analogy between value attribution in co-operative game theory and feature importance assignment in machine learning models. SHAP operates by assessing the marginal contribution of a chosen feature to all possible combinations of features which do not contain an input of interest. However, doing so is computationally expensive, prompting the development of approximation algorithms. An example of this is KernelSHAP \cite{lundberg2017unified}, which is shown to converge to SHAP values as the sample space approaches the set of all possible coalitions. 

Gradient based methods are also popular feature importance techniques that measure the sensitivity of the output to perturbations of the input. We may use vanilla gradients for a particular sample \cite{simonyan2013deep}, or use regularization techniques such as in SmoothGrad \cite{smilkov2017smoothgrad} or integrated gradients \cite{sundararajan2017axiomatic}. These methods tend to be less expensive than computing SHAP values.  \cite{brigo2021interpretability} applies a range of these techniques, as well as LIME explanations \cite{ribeiro2016should}, in the context of financial time series. 

Indeed, much research has gone into speeding up the computation of SHAP values. For example, FastSHAP uses a deep neural network to approximate the SHAP imputations \cite{jethani2021fastshap}. While this approach does not maintain the desirable theoretical guarantees of SHAP, the technique is fast, and generally accurate. However, since this work relies on neural networks, it raises the potential challenge of having to explain the explanation.

In particular, with respect to Recurrent Neural Networks, many of the above model agnostic techniques apply. However, there are certain methods that are specific to the time series domain and to certain model architectures. \cite{murdoch2018beyond} leverage a decomposition of the LSTM function to compute the relevance of a particular feature in a given context. TimeSHAP \cite{bento2021timeshap}, the closest work to our paper, is an algorithm to extend KernelSHAP to the time series domain, by selectively masking either features, or time steps. 

There are other ways of explaining time series predictions. Computing neural activations for a range of samples \cite{karpathy2015visualizing} finds that certain neurons of neural networks act as symbols for events occurring in the data. \cite{arras2017explaining} apply Layerwise Relevance Propagation (LRP) \cite{bach2015pixel} to recurrent neural networks.

Finally, this paper makes use of traditional time series modelling techniques, including autoregressive models (AR), moving average models (MA), autoregressive moving average models (ARMA) and vector autoregressive moving average models with exogenous variables (VARMAX). We point to \cite{shumway2000time} for a general introduction to the topic and \cite{milhoj2016multiple} for specifics on the VARMAX model and how to calibrate or train these models. 

\section{SHAP for Time Series}
\subsection{SHAP and KernelSHAP}
Let the supervised learning problem on time series be defined by $\mathcal{D} = (\mathcal{X}, \mathcal{Y})$, the data of the problem on $\mathcal{X} = \{X\}_{j\in I} \subset \mathbb{R}^{N \times W}$, $\mathcal{Y} = \{y\}_{j\in I} \subset \mathbb{R}^M$, where $X$ is a matrix of $W$ column vectors (each column represents the $N$ features), one for each step in the lookback window, $I$ is an indexing set, and $|I|$ is the number of samples in the dataset, $W$ is the size of the window, $N$ is the number of features shown to the network at each time step $w \in [W]$, which represents $w \in \{1,...,W\}$. Finally, $M$ is the dimensionality of the output space. 

A function approximating the relationship described by $\mathcal{D}$ and parametrised by $\theta \in {\Theta}$, a parameter space, is of type $f_\theta : \mathbb{R}^{N \times W} \rightarrow \mathbb{R}^M$. In particular, we let $f_\theta$ be a recurrent neural network, such as an LSTM \cite{hochreiter1997long}. 

The formula for the SHAP value \cite{lundberg2017unified} of feature $i\in C$, where $C$ is the collection of all features, is given by: 
\[
\phi_{v}(i) = \sum_{S \in \mathcal{P}(C)\setminus\{i\}} \frac{(N-|S|+1)! |S|!}{N!} \Delta_v(S,i),
\]

where $\mathcal{P}(C)$ is the powerset of the set of all features, and, for a value function $v: \mathcal{P}(C) \rightarrow \mathbb{R}$, the marginal contribution $\Delta_{v}(S,i)$ of a feature $i$ to a \textit{coalition} $S\subset \mathcal{P}(C)$ is given by 
\[
\Delta_v(i,S) = v(\{i\} \cup S ) - v(S).
\]

Even for small values of $|C|$, $|\mathcal{P}(C)| = 2^{|C|}$ is large, implying that the SHAP values cannot be easily computed. KernelSHAP, again presented in \cite{lundberg2017unified} is a commonly used approximation of SHAP, which provably converges to SHAP values as the number of perturbed input features approaches $|\mathcal{P}(C)|$. More precisely, we define KernelSHAP as the SHAP values of the linear model $g$ given by the minimization of

\begin{equation} \label{kernel}
\min_{g \in LM} \sum_{z\in \mathcal{Z}} (f_\theta(h_x(z)) - g(z)) \pi_x(z),
\end{equation}

where $h_x: \mathcal{Z} \rightarrow \mathbb{R}^d$ is a masking function on a $d-$dimensional $\{0,1\}$-vector $z$ belonging to a sample $\mathcal{Z}\subseteq \{0,1\}^d $, the collection of all possible said vectors, each representing a different coalition. In practice, this function maps a coalition to the masked data point $x^\star$, on which we compute the prediction $f_\theta(h_x(z))$. Finally, $\pi_x$ is combinatorial kernel, from which the method gets its name, given by: 
\begin{align*}
\pi_x(z) = (d-1)\left({{d\choose |z|}|z|(d-|z|)}\right)^{-1}
\end{align*}
which entails an optimization under a weighted ordinary least square minimization and $|z| = \sum_{i = 1}^d z_i$. 

\subsection{SHAP Values for Time Series Models}

In this section we prove three results. First, we verify the applicability of KernelSHAP in the time series domain. Indeed, using a linear model is inadequate to account for the time component. However, if instead of fitting a local linear surrogate we fit a Vector Autoregressive (VAR) model, we show that the same approximation properties are maintained. 

Secondly, since computing SHAP values in general is computationally intensive, we find the explicit form of SHAP values for several well known time series models. Finally, we will present a strategy to detect important events in the data through feature importance. 

\begin{proposition}
Vector autoregressive SHAP converges to SHAP Values as the number of samples approaches $|\mathcal{P}(n)|$, where $n$ is the number of features.  
\end{proposition}

\begin{proof}
Coefficients of vector autoregressive models are estimated through Ordinary Least Squares. This means that a vectorised version of VAR can be used to estimate the coefficients in the same way as the linear model. Therefore, Theorem 2 from \cite{lundberg2017unified} carries over to the VAR case. 
\end{proof}

This amounts to replacing $g$ with a VAR model, instead of a linear model, and understanding that $z \in \mathcal{Z}\subset \{0,1\} ^{N \times W}$ in Equation \ref{kernel}. 

Closed form solutions are known for the SHAP values of OLS linear regression \cite{lundberg2017unified}. In what follows, we present the closed form solution for AR, MA and ARMA models first, and then the broader family of VARMAX models.

In the following three propositions we will let \\$y_t = f(y_{t-1}, y_{t-2}, ..., y_{t-W})$,
while making use of the now commonly used abuse of notation $f(y_{S}, y_{\overline{S}})$, to represent the function $f$ where all $y_s, s \in S$ are feature values and $y_{s'}, s' \in \overline{S}$ have been masked. 

\begin{proposition}
Suppose $y_t$ is an $AR(W)$ process,

\[ y_t =\alpha + \sum_{w = 1}^{W} \beta_{t-w}\cdot y_{t-w}  + \epsilon_t\]

The Shapley values of feature $y_{t-w}$ for prediction $y_t$ are given by: 

\[\phi_f^t(t-w) = \beta_{t-w}( y_{t-w} - y_{t-w}^\star), \]

where $y_{t-w}^\star$ is the masked feature $y_{t-w}$.
\end{proposition}

\begin{proof}
 
Recall that the Shapley values are obtained via
\begin{equation}
\begin{aligned}
\phi_f(i) & =  \sum_{S\in \mathcal{P}[W]/i} \frac{|S|! ( N - |S| - 1)!}{N!} \Delta_f( S, i ), \\
\end{aligned}
\end{equation}

where we adopt the notation that $[W] = \{1,2,...,W\}$ We now find the quantities $ \Delta_f( S, i )$, where we take $i:=t-w$. By definition $i\in \overline{S}$. We have that

\begin{align*}
\Delta_f(S,i) & = f\left(y_{S\cup \{i\}}, y_{\overline{S \cup \{i\}}}\right) - f(y_{S}, y_{\overline{S}}) \\ 
& = \sum_{s \in S\cup \{i\}} (\beta_s y_s)  
+ \sum_{s \in \overline{S \cup \{i\}}} (\beta_s y^\star_s) 
+ \alpha \\ &
-  \sum_{s \in S} (\beta_s y_s) 
- \sum_{s \in \overline{S}} (\beta_s y^\star_s) 
- \alpha \\
&=\beta_{i} y_{i} 
+ \sum_{s \in S} (\beta_s y_s) 
+ \sum_{s \in \overline{S \cup \{i\}}} (\beta_s y^\star_s)  \\ &
- \sum_{s \in S} (\beta_s y_s) 
-  \sum_{s \in \overline{S \cup \{i\}}} (\beta_s y^\star_s) 
- \beta_iy^\star_i\\
& = \beta_i (y_i - y^\star_i) \\
\end{align*}
for all $S \in \mathcal{P}([W]/i)$. Since the quantity $\Delta_f(S, i )$ is independent of $S$, that is $\Delta_f(S, i )=\Delta_f(S', i )$ for all $S,S'\in \mathcal{P}([W]/\{i\})$ we have that
\begin{equation}
\begin{aligned}
\phi_f(i) & = \Delta_f(\cdot, i ) \sum_{S\in \mathcal{P}[W]/i} \frac{|S|! ( N - |S| - 1)!}{N!}  \\
& = \beta_i (y_i - y^\star_i)\\
\end{aligned}
\end{equation}
\end{proof}

\begin{proposition}
Suppose $y_t$ is an $MA(W)$ process, 

\[y_t = \alpha + \epsilon_t + \sum_{w = 1}^{W} (\gamma_{t-w} \epsilon_{t-w}).\]
The Shapley values of feature $y_{t-w}$ for prediction $y_t$ are given by: 

\[\phi_f^t(t-w) = c^\sharp_{t-w} (y_{t-w} - y_{t-w}^\star), \]

where $y_{t-w}^\star$ is the masked feature $y_{t-w}$, and where $c_i^\sharp = \gamma_i $ when $i \in \{t-W-1, t-1\}$ or $c_i^\sharp = \gamma_{i+1}- \gamma_{i}$ otherwise.
\end{proposition}

\begin{proof}
Note that we can rewrite a moving average process as: 
\begin{align*}
y_t & = \alpha + \epsilon_t + \sum_{w = 1}^{W} (\gamma_{t-w} \epsilon_{t-w})\\
& = \alpha + \epsilon_t + \sum_{w = 1}^{W} \gamma_{t-w} (y_{t-w} - y_{t-w-1})\\
& = \alpha + \epsilon_t + \sum_{w = 1}^{W} \gamma_{t-w} y_{t-w} - \gamma_{t-w} y_{t-w-1} \\
& = \alpha + \epsilon_t + \gamma_{t-1} y_{t-1} - \gamma_{t-W} y_{t-W-1} + \sum_{w=2}^{W} \gamma_{t-w} y_{t-w} - \gamma_{t-w+1} y_{t-w}\\
& = \alpha + \epsilon_t  + \gamma_{t-1} y_{t-1} - \gamma_{t-W} y_{t-W-1} + \sum_{w=2}^{W} (\gamma_{t-w} - \gamma_{t-w+1}) y_{t-w} \\
& = \alpha + \epsilon_t  + \gamma_{t-1} y_{t-1} - \gamma_{t-W} y_{t-W-1} + \sum_{w=2}^{W} c_{t-w} y_{t-w}, 
\end{align*}
where we call $c_{t-w}:= \gamma_{t-w}- \gamma_{t-w+1}$. For notational convenience let us define

\[   
c^\sharp_i := 
     \begin{cases}
       \gamma_i & \text{if } i \in \{ t-W-1, t-1\}; \\
        c_i & \text{otherwise.} \\
     \end{cases}
\] We can compute the marginal contribution as follows: 
\begin{align*}
\Delta_f(S,i) & = f\left(y_{S\cup \{i\}}, y_{\overline{S \cup i}}\right) - f(y_{S}, y_{\overline{S}}) \\ 
& = \alpha + \epsilon_t  
+ \sum_{s \in S\cup \{i\}} c^\sharp_{t-s} y_{t-s} 
+ \sum_{s \in \overline{S\cup \{i\}}} c^\sharp_{t-s} y^\star_{t-s} \\&
- \alpha - \epsilon_t  
- \sum_{s \in S} c^\sharp_{t-w} y_{t-w} 
- \sum_{s \in \overline{S}} c^\sharp_{t-s} y^\star_{t-s} \\
& = c^\sharp_i (y_i - y^\star_i), \\
\end{align*}
where we attain the simplification in the last line in the same way as the previous Lemma.
Since $\Delta(S,i)$ is independent of $S$, we have

\[\phi_f(i) = \Delta(\cdot,i) \]
as required. 
\end{proof} 

\begin{proposition}
Suppose $y_t$ is a stationary $ARMA(p,q)$ process, with $p,q \geq 1$: 
\begin{align*} & = \alpha + \epsilon_t 
+ \sum_{w = 1}^{p} \beta_{t-w}\cdot y_{t-w}  
+ \sum_{w = 1}^{q} \gamma_{t-w} \cdot \epsilon_{t-w}.\end{align*}

For any window size $w$, the Shapley values of feature $y_{t-w}$ for prediction $y_t$ are given by 
\begin{align*} \phi_f^t(t-w) &= c^\flat_{t-w} (y_{t-w}- y_{t-w}^\star ),
\end{align*}
where $c^\flat_{t-w}$ is the coefficient of feature $y_{t-w}$ of the model, as described in the proof, and where $y_{t-w}^\star$ is the masked feature $y_{t-w}$.
\end{proposition}

\begin{proof}
In a similar fashion to how SHAP values for the $\text{MA}(W)$ model were determined, we need to distinguish different cases for the coefficients associated to a particular feature. We will distinguish between 5 cases, numbered in the equation below. 

A given feature $i$ may be taken in consideration by the AR part of the ARMA model (case 1), the MA part (cases 2,3), or possibly both (cases 4,5). When the feature has moving average coefficients associated to it, there may be either a $\gamma_i$ (cases 2, 4) or a difference  $\gamma_{i+1}-\gamma_i$ (cases 3,5), when $i \not \in \{t-1, t-q-1\}$. Explicitly, the case-by-case value for $c^\flat_i$ is

\[   
c^\flat_i = 
     \begin{cases}
        \beta_i & \text{if case 1, with } p \geq i > q+1 > 1, \\
        \gamma_i &  \text{if case 2, with } i = q+1 \: \text{and} \: q+1 > p,  \\
        \gamma_{i+1} - \gamma_i &  \text{if case 3, with } q+1 > i \geq p, \\
        \beta_i + \gamma_i &  \text{if case 4, with } i= 1, \: \text{or} \: i = q+1 \: \text{with}  \: q+1 < p, \\
        \beta_i + \gamma_{i+1} - \gamma_i & \text{if case 5, with } 1 < i < q+1 \: \text{and} \: 1 < i < p. \\
     \end{cases}
\]
The marginal contribution of feature $i$ to the coalition $S$ is given by: 
\begin{align*}
\Delta_f(S,i) & = f\left(y_{S\cup \{i\}}, y_{\overline{S \cup i}}\right) - f(y_{S}, y_{\overline{S}}) \\ 
& = \alpha + \epsilon_t  
+ \sum_{s \in S\cup \{i\}} c^\flat_{t-s} y_{t-s} 
+ \sum_{s \in \overline{S\cup \{i\}}} c^\flat_{t-s} y^\star_{t-s}
- \alpha - \epsilon_t  \\
& - \sum_{s \in S} c^\flat_{t-w} y_{t-w} 
- \sum_{s \in \overline{S}} c^\flat_{t-s} y^\star_{t-s} \\
& = c^\flat_i (y_i - y^\star_i), \\
\end{align*}
which, as we have seen in the previous proofs, does not depend on $S$ and is therefore equal to $\phi^t_f(i)$, as required. 
\end{proof}

We proceed extending the above to the broader family of VARMAX models. 
\begin{proposition}
Suppose $\vec{y_t}$ is a \[\text{VARMAX}(\{\mathbf{A}\}_{i\in [P]},\{\mathbf{M}\}_{j \in [Q]}, \{\mathbf{B}\}_{l \in [L]}
)\]
process, with $\mathbf{A}_i, \mathbf{M}_j, \mathbf{B}_l
\in \mathbb{R}^{n\times n}$, the matrices of coefficients. VARMAX can be written as: 

\begin{align*}\vec{y}_t &= \vec{\alpha} + \vec{\epsilon}_t 
+ \sum_{w = 1}^{P} \mathbf{A}_{w}\cdot \vec{y}_{t-w} + \sum_{w = 1}^{Q} \mathbf{M}_{w} \cdot \vec{\epsilon}_{t-w} + \sum_{w= 1} ^{L}\mathbf{B}_w\vec{x}_w.\end{align*}

The Shapley values of feature $\mathbf{y}_{t-w}$ for prediction $\mathbf{y}_t$ are given by: 

\[
\phi_f^t(i) = 
     \begin{cases}
        \mathbf{B}_{i,i}\cdot (x_i-x_i^{\star}) & \text{if $i$ exogenous}\\
        \mathbf{C}_{i,i}\cdot (y_i - y_i^{\star}) & \text{if $i$ endogenous},
     \end{cases}
\]
where $\vec{y}_{i}^\star, \vec{x}_{i}^\star$ is the masked feature $\vec{y}_{i}. \vec{x}_{i}$.
\end{proposition}

\begin{proof}
The proposition follows a similar structure as with the ARMA model, with the key difference that we will need to select a feature with a time and a space index. Indeed, the vector 

\[   
\mathbf{C}_w= 
     \begin{cases}
        \mathbf{A}_w & 1 > p \geq w > Q+1 \\
        \mathbf{A}_w + \mathbf{M}_w & w= 1, \: \text{or} \: w = Q+1 \: \text{with}  \: Q+1 < P, \\
        \mathbf{M}_{w+1} - \mathbf{M}_i & Q+1 > i \geq P \\
        \mathbf{M}_w & w = Q+1 \: \text{with} \: Q+1 > P \\
        \mathbf{A}_w + \mathbf{M}_{w+1} - \mathbf{M}_w &  1 < w < Q+1 \: \text{and} \: 1 < w < P. \\
     \end{cases}
\]
We can check that, as before: 
\begin{align*}
\Delta_f(S,i) & = \mathbf{C}_w (\vec{y}_w - \vec{y}^\star_w)  + \mathbf{B}_l(\vec{x}_l - \vec{x}_l^{\star}). \\
\end{align*}

Depending now on whether the feature $i$ is endogenous or exogenous, we can complete the proof by setting $\mathbf{B}_i,\mathbf{C}_i$ to be the $i$th column of the matrices $\mathbf{B}$ , $\mathbf{C}$ respectively: 

\begin{align*}
\Delta_f(S,i) = 
     \begin{cases}
        [0,...,0, (x_{w,i}-x_{w,i}^{\star}),0,...,0] \cdot \mathbf{B}_i & \text{if exogenous}\\
         [0,...,0, (y_{w,i} - y_{w,i}^{\star}),0,...,0] \cdot \mathbf{C}_i & \text{if endogenous},
     \end{cases}
\end{align*}
where $y_{w,i}$ represents the value of the time series $i$ at $w$ time steps ago. The above simplifies to: 
\[   
\Delta_f(S,i) = 
     \begin{cases}
        \mathbf{B}_{i,i}\cdot (x_i-x_i^{\star}) & \text{if exogenous}\\
        \mathbf{C}_{i,i}\cdot (y_i - y_i^{\star}) & \text{if endogenous},
     \end{cases}
\]

which is independent of $S$ and, hence, equivalent to SHAP values. 
\end{proof}

Notice that in certain cases it is possible to decouple the moving average and autoregressive parts of the SHAP values. For instance, for the ARMA model, because $c_i^\flat = \beta_i + \gamma_i$ whenever $i = 1 \text{ or } i = q \text{ with } q + 1 < p$, and in that case $\phi_f(i) = (\beta_i + \gamma_i) \cdot (y - y^\star)$. It is easy to see how we can deconstruct the SHAP values in this case to provide further insight on the attribution for the model. We leave these analyses for future work.

\subsection{Event Detection with SHAP Values}
Local feature importance metrics can provide insight into which time step and feature in a window of past values was important to the current prediction. We are often analysing multiple consecutive predictions, where the windows will of course overlap. We seek a method of consolidating several overlapping windows of feature allocations to see how influential a given time localised \textit{event} (literally, the feature values at a particular time step) is in the model's predictive behaviour.

Formally, we recall that $y_t = f(y_{t-1}, ..., y_{t-W})$ and notice that $y_{t-1} = f(y_{t-2},..., y_{t-W})$. Therefore, when we compute the SHAP values for $y_{t-2}$, for example, we will get the importance that the feature had in the first prediction and in the second. We repeat this process for all windows in which $y_{t-2}$ appears. After summing over all these collected terms we get our event detection value for $y_{t-2}$. 

In short, this methodology allows us to say how important time step $y_{t-2}$ was for \textit{any prediction}. We expect this to spike when $y_{t-2}$ is an important event, affecting multiple predictions in time to a strong extent. In Figure \ref{fig:plot_event}, we show an example of the consolidation we describe above. In Section \ref{sec:results} we go through the analysis in more detail.

\section{Time Consistency for Shapley Values}
As we have described above, SHAP operates from an analogy between game theory and deep learning; 
namely, we interpret features of a model as players in a cooperative game. When several parties collaborate, computing the contribution
that each of these brings towards a final goal can help 
the parties distribute payoffs. Hence, Shapley Values can be used
to find fair allocations of rewards in the context of a co-operative game. For games that develop in time, however, these allocations
may not provide sufficient incentive for all parties to 
pursue an initial goal over time, and it may be optimal for a
player to deviate from the co-operative strategy. In order to avoid 
said deviations, game theorists study imputation schedules 
that ensure an optimal strategy for all parties involved. 
In particular, the notion of \textit{time consistency} is used
for example in \cite{petrosjan2003time} or \cite{reddy2013time} to 
manage incentives across time. 

In the context of feed-forward neural network, where the input features exist in distinct spaces, features can reasonably be considered as different players in a cooperative game. 
However, in the time series domain, the same reasoning will lead SHAP to 
consider the same feature at different
time steps as different players in a game.
In light of this, we think it is appropriate to extend the 
game theoretical analogy of SHAP to Time Consistent Shapley values, as we
believe this could lead to a more principled approach of SHAP in the context
of time series. 

In \cite{petrosjan2003time}, the authors present time consistency of Shapley Values in the form of: 
\begin{equation}\label{tcs_eq}
\sum_{w = 0 }^{t-1} \beta ( w,i ) + \phi (t,i)  = \phi (0, i)
\end{equation}

where $\beta$ represents an imputation schedule of payments made to player $i$ across time steps $t$. 
Note that here $\phi(0,i)$ is taken to be the total value that players $i$ contributes to the game. 
A motivating example is that of an oligopoly, in which collaborating players agree to a strategy but need an incentive scheme 
that will make deal renegotiation or dropping out undesirable. In that case, 
an initial sum is agreed upon based on the value each contributor will bring to the game, 
and dividends are paid according to marginal contributions to the subgames.

Computing these values in the context of a prediction problem can be done through the following steps: 
\begin{itemize}
    \item Step 1: Compute total SHAP of features $i$ by masking intermittently their visible history. In \cite{bento2021timeshap} this is referred to as "Feature Shap". 
    \item Step 2: For $w \in [W]$, compute what we refer to as \textit{subgame SHAP} for time steps $t-w$ by fixing the interval starting at $t-W$ and finishing at $t-w$ to their observed value, and masking the remaining intervals. We do so according to the different coalitions of players $ S \in \mathcal{P}([N])/\{i\}$. 
    \item Step 3: Compute the imputation schedule as defined in Equation \ref{tcs_eq}.
\end{itemize}

This procedure is clearly computationally demanding; yet, it requires fewer perturbations of the input than KernelSHAP. In fact, the total number of coalitions out of which we sample are $W * 2^N$ as opposed to $2^{W\times N}$. Moreover, the process overall is parallelizable as we are fitting a linear model for each sub-window.
Notice that in the presence of large lookback windows, approximating SHAP values 
through KernelSHAP \cite{lundberg2017unified} becomes infeasible. This is due to the 
combinatorial term in the kernel, which explodes, leading to numerical underflow
or near zero regression weights. Time Consistent SHAP can 
scale with the size of the window, extending the state of the art.
We operate in a setting that is simpler than the general multi-decision 
game on a tree. Therefore, imputations are simplified. 
\begin{proposition}
Time Consistent imputations for Shapley Values are determined by the equation: 
\begin{align*}
\beta(t,i) = \phi(t,i) - \phi(t+1,i) 
\end{align*}
\end{proposition}

\begin{proof}

First rewrite the time consistency equation as: 
\[
\phi(t, i) = \phi(0,i) - \sum_{w = 0}^{t-1} \beta(w,i).  
\]

Notice that: 
\begin{align*}
    &\phi(t+1,i) - \phi(t, i) = \phi(0,i) - \sum_{w = 0}^{t} \beta(w,i) - \phi(0,i) + \sum_{w = 0}^{t-1} \beta(w,i) \\
    \implies & \phi(t+1, i) - \phi(t,i) =  \sum_{w = 0}^{t-1} \beta(w,i) - \sum_{w = 0}^{t} \beta(w,i) = - \beta(t,i) \\
    \implies & \beta(t,i) = \phi( t,i) - \phi(t+1,i).\\ 
\end{align*}
\end{proof}

\section{Experiments}
\subsection{Methodology}

Our experiments present a comparison between VARSHAP and Time Consistent SHAP. The former is the time series equivalent of KernelSHAP; hence, we will use it as a baseline with which to compare the novel technique of time-consistent SHAP. Once we have computed VARSHAP values, we will apply our event detection algorithm to identify spikes. 

Our experiments were run on the Individual Household Electricity dataset.\footnote{https://archive.ics.uci.edu/ml/datasets/individual+household+electric+power+consumption, dataset is licensed under Creative Commons Attribution 4.0 International (CC BY 4.0).} The choice of this dataset is due to our ability to reason about specific events that the recurrent network give importance to: in particular, the turning on and off of appliances in a room (as captured by a submetering feature). In this context, we expect our proposed event detection to spike for a given meter. 

We only consider the first $10^6$ data points, for ease of computation. Windows of size $10$ are computed and the train-test split is of $0.9$. We train a small recurrent neural network with an LSTM layer followed by a dense layer for $1000$ epochs with batch size of $200$. The out-of-sample performance in terms of mean squared error is approximately $5 \times 10^{-4}$. 

In both Time Consistent SHAP and VARSHAP we mask the values excluded from the coalition with zero fillers. We sample $1000$ perturbations for both techniques. 

\subsection{Results}
\label{sec:results}
\begin{figure*}[h]  \centering  \includegraphics[width=\textwidth]{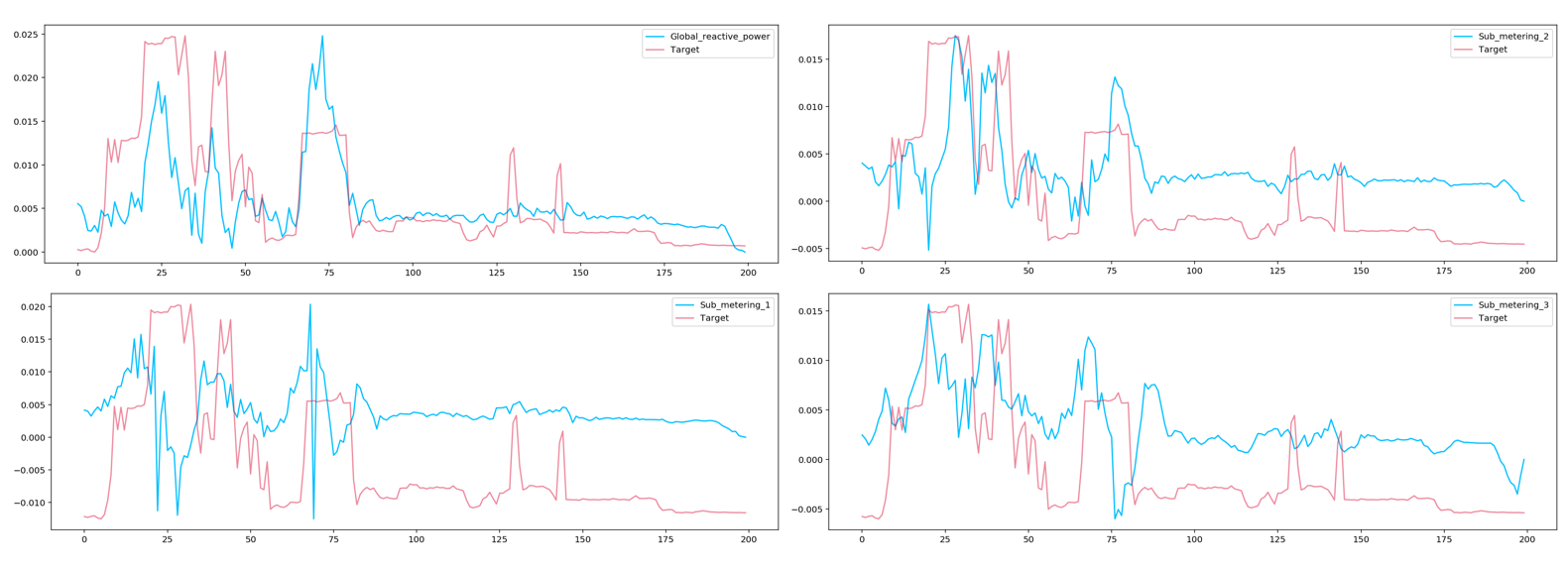}  \caption{Event detection plots, where we overlay the scaled target (in red) on top of the event detection convolution for selected features. The plateauxs in the target variables may be interpreted as appliances being turned on or off in a given room, as portrayed by a submetering.}  \label{fig:plot_event}\end{figure*}

Our experiments provide insight on the difference between the two techniques: VARSHAP and Time Consistent SHAP values. Our first observation is the striking difference between the results, highlighting the importance of a sampling strategy in SHAP computations. VARSHAP attains a smaller range of values compared to Time Consistent SHAP values, and spikes less dramatically, especially for features farther away in the past. The overall shape of the curve is similar for all time steps. In contrast, Time Consistent SHAP values jump at different locations depending on which time step we consider and can attain a large range of strengths for far in the past. 

\begin{figure*}[h]  \centering  \includegraphics[width=\linewidth]{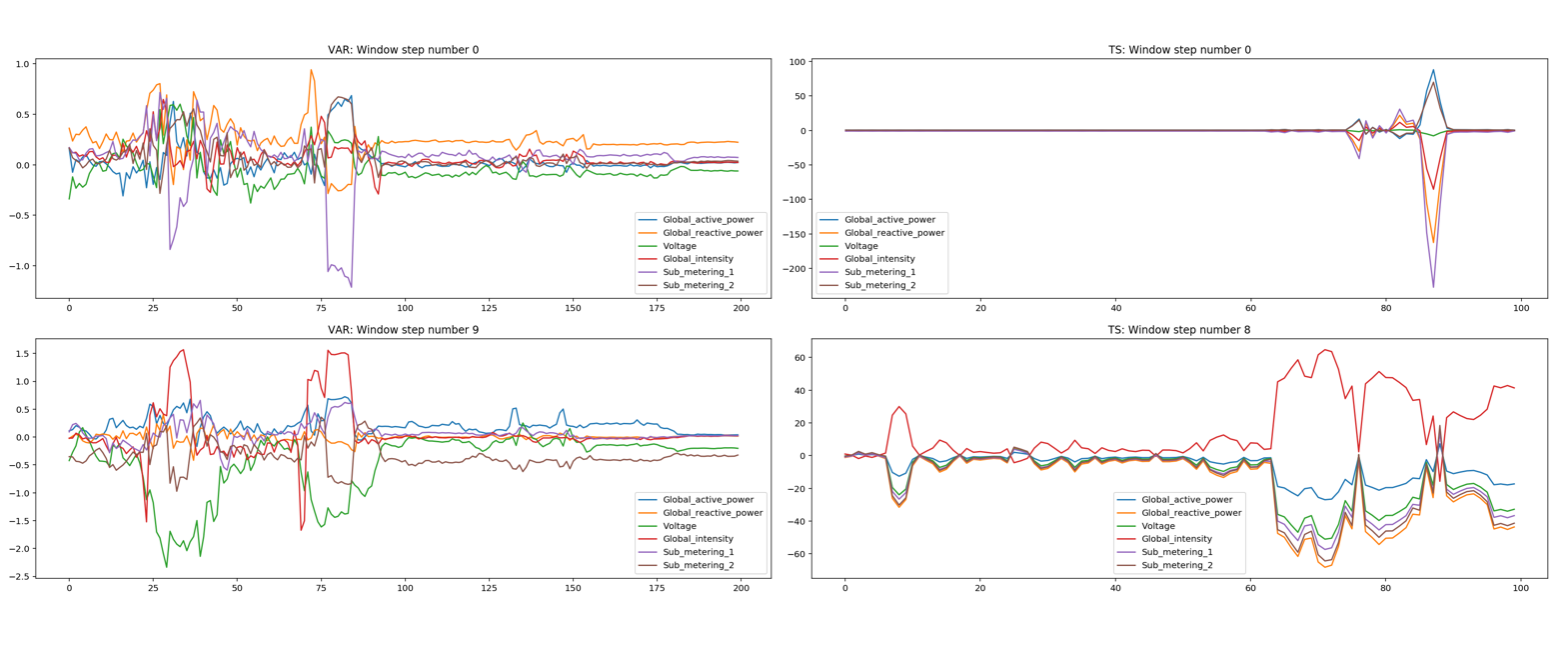} \caption{We compare VARSHAP with Time consistent Shap Values for the first and last window, on all features, for the same time interval. Window 0 is further in the past with respect to the prediction. The two methodologies present strikingly different results. }\label{fig:plot_tcs} \end{figure*}

We recall that the core difference between the two algorithms resides in how we compare marginal contributions of the features. In VARSHAP we consider the effect of a feature-time step pair to all other such pairs, whereas in Time Consistent SHAP we only assess how, at a given time step, the evolution of a feature will affect coalitions of other feature trajectories. 

Therefore, we can interpret the feature importance in the least recent time step as how important a feature is to the prediction overall, whereas we can interpret the most recent time step as how sensitive the prediction was to the last time step. From this point of view, the large spikes of feature importance at the beginning of the window in Figure \ref{fig:plot_tcs} signify a strong sensitivity to the features overall. Overall, it is interesting to see how different time steps of the window can provide a varied picture of the model behaviour in the case of time consistent SHAP, whereas VARSHAP tends to confirm similar results in different time steps. 

Moreover, the two measures generally disagree on which prediction has larger attributions. This is due to the fact that KernelSHAP values are denominated in the difference between the sample's predicted value and the the masked prediction. It is therefore expected that the sums of SHAP values would be different in the two computations. 

Our event detection in Figure \ref{fig:plot_event} highlights how different jumps in the target can be explained by activity in different appliances, as seen by the submeterings in different rooms. Indeed, the first surge seems to be caused by submetering 1, which grows at the beginning. A second surge, at the 23rd time step is due to an appliance in the second submetering, while the importance of the first submeter decreases simultaneously. Overall, we can attribute most jumps to a certain submeter, and generally not to the others, which aligns with our goal of detecting events.

Overall, our experiments underscore a broader problem of difficulty in synthesizing information from SHAP value computation. Indeed, because of the high dimensional nature of time series data, the SHAP values are themselves difficult to compare. If we were to consider sample statistics these would be not descriptive of the overall behaviour of the SHAP values across the data set. This issue is exacerbated by the size of our data set which forces us to only show a portion of the computed SHAP values.   

\section{Conclusion}
We have investigated the issue of computing feature importance values for time series data. In particular, we provide closed form solutions for the SHAP values of a number of time series models. 
We have also presented two time-series focused feature importance techniques that we have developed: VARSHAP, and Event Detection. VARSHAP is the time series equivalent to KernelSHAP.

We leave to future investigation many deployment details. In particular, the ability to decompose SHAP values for ARMA models into a moving average and an autoregressive part suggest that we can refine the feature attribution to gain greater insight on the behaviour of the model. For example, within the LIME framework \cite{ribeiro2016should}, by fitting an ARMA surrogate it may be possible to decouple at each time step the importance due to trend (moving average part) from the time step specific effect (autoregressive part).

Many properties of Time Consistent SHAP values need to be explored. We noticed that the new strategy to compute feature importance for time series, which follows a principled approach from the game theory literature, yields dramatically different results from VARSHAP. Future work should enhance this comparison, understanding in greater detail where each algorithm's strength resides. A finer comparison would help identify the optimal strategy to deploy for different time series. Indeed, one such comparison would be through measures of stability, faithfulness and robustness \cite{zhou2021evaluating}. 

KernelSHAP and several related strategies leverage the assumption of feature independence to approximate Shapley Values \cite{lundberg2017unified}. Indeed, this is also the case with VARSHAP; however, in the time series domain autocorrelations (the correlation between the same feature in different time steps) are commonly present. It is possible that Time Consistent SHAP values overcome this issue, since it avoids a temporal comparison between features, but further theoretical analysis is required to ascertain this.

Feature importance methods provide a first probe into the behaviour of neural networks; their increasing popularity due to the growing availability of Python implementations provides a stepping stone for practitioners into interpretability of machine learning models. Especially in light of the complexities in assessing explanations, building explanations that are tailored to the known assumptions in the problem, in our case the existence of a time component, may increase our confidence in their deployment. Overall, model-agnosticism may lead to unrefined explanations that fail to capture important aspects of the model behaviour. Having explored this problem in the setting of time series, we showed above all that there are many open questions to be addressed on how to appropriately adapt existing XAI techniques to complex families of architectures.
\section{Acknowledgements}
This paper was prepared for informational purposes by the Artificial Intelligence Research group of JPMorgan Chase \& Co and its affiliates (``J.P. Morgan''), and is not a product of the Research Department of J.P. Morgan.  J.P. Morgan makes no representation and warranty whatsoever and disclaims all liability, for the completeness, accuracy or reliability of the information contained herein.  This document is not intended as investment research or investment advice, or a recommendation, offer or solicitation for the purchase or sale of any security, financial instrument, financial product or service, or to be used in any way for evaluating the merits of participating in any transaction, and shall not constitute a solicitation under any jurisdiction or to any person, if such solicitation under such jurisdiction or to such person would be unlawful.   
\bibliographystyle{alpha}
\bibliography{sample-base}
\end{document}